\documentclass[10pt,twocolumn]{article}

\usepackage[left=.75in, right=.75in, top=1in,bottom=1in]{geometry}
\usepackage{amssymb,amsmath,amsthm}

\usepackage{gensymb}
\usepackage[font={small,sf,bf}]{caption}
\usepackage{xfrac}  
\usepackage{empheq}  
\usepackage[hidelinks]{hyperref}
    \hypersetup{colorlinks=false}
\usepackage{cleveref}  
\usepackage{lipsum}  
\usepackage{xcolor}  
    \definecolor{highlight}{rgb}{1,1,0}
\usepackage{authblk}
\usepackage{graphicx}  
\usepackage{subfig}
\usepackage{grffile}  
\usepackage{tabularx}  
    \newcolumntype{Y}{>{\centering\arraybackslash}X}
\usepackage{colortbl}  
\usepackage{xspace}  
\usepackage{setspace}  
\usepackage{algorithm}  
\usepackage[noend]{algpseudocode}  
    \algrenewcommand\alglinenumber[1]{{\sffamily\color[rgb]{0.15, 0.35, 0.9}\footnotesize#1}}  
    \algrenewcommand{\algorithmiccomment}[1]{\hfill #1}  
    
    \algblock{ParFor}{EndParFor}
    \algnewcommand\algorithmicparfor{\textbf{parfor}}
    \algnewcommand\algorithmicpardo{\textbf{do}}
    \algnewcommand\algorithmicendparfor{\textbf{end\ parfor}}
    \algrenewtext{ParFor}[1]{\algorithmicparfor\ #1\ \algorithmicpardo}
    \algtext*{EndParFor}{} 
    
    \algblock{Streaming}{EndStreaming}
    \algnewcommand\algorithmicstreaming{\textbf{streaming}}
    \algnewcommand\algorithmicstreamingdo{\textbf{}}
    \algnewcommand\algorithmicendstreaming{\textbf{end\ streaming}}
    \algrenewtext{Streaming}[1]{\algorithmicstreaming\ #1\ \algorithmicstreamingdo}
    \algtext*{EndStreaming}{} 
\usepackage{xpatch}  
    \makeatletter
    \xpatchcmd{\algorithmic}{\itemsep\z@}{\itemsep=2 pt}{}{}
    \makeatother
\usepackage{tikz}
    
    \usetikzlibrary{shapes.misc,positioning,fit,calc}
    
\usepackage[version=4]{mhchem}  
\usepackage{cuted}  
\usepackage{booktabs}  
\usepackage{multirow}  
\usepackage{makecell}  
    
\usepackage{enumitem}  
    \setlist{nosep}  
\usepackage{placeins} 
\usepackage{adjustbox}  
\usepackage{soul}  
\usepackage{etoolbox}

\newtheorem{theorem}{Theorem}
    \AfterEndEnvironment{theorem}{\noindent\ignorespaces}

\makeatletter
    \newcommand{\smallotimes}{\mathbin{\mathpalette\make@small\otimes}}
    \newcommand{\make@small}[2]{%
    \vcenter{\hbox{%
        \scalebox{0.6}{$\m@th#1#2$}%
    }}%
    }
    \let\sv@thm\@thm
    \def\@thm{\vspace{0.75em}\let\indent\relax\sv@thm}
\makeatother

\providecommand{\keywords}[1]{{\bfseries\noindent Keywords: }\xspace{#1}}

\usepackage{pxfonts}

\usepackage[normalem]{ulem}

\newcommand{\WISHLIST}[1]{}
\newcommand{\REM}[1]{}

\newcommand{\B}[1]{\mathbf{#1}}
\newcommand{\BB}[1]{\boldsymbol{#1}}
\newcommand{\td}[1]{\tilde{#1}}
\newcommand{\T}[0]{\mathsf{T}}
\newcommand{\Det}[1]{\left| #1 \right|}
\newcommand{\diag}[1]{\mathrm{diag}\left( #1 \right)}
\newcommand{\DIAG}[1]{\mathbf{diag}\left( #1 \right)}

\newcommand{\D}{\partial}
\newcommand{\Cov}[2]{\mathrm{Cov}\left[ #1, #2 \right]}
\newtheorem{corollary}{Corollary}[theorem]
\newcommand{\lp}[0]{\ell_p}
\newcommand{\LL}[0]{\mathcal{L}}
\newcommand{\KK}[0]{\BB{\mathcal{K}}}
\newcommand{\KKinv}[0]{\BB{\mathcal{K}}^{-1}}
\newcommand{\KKtinv}[0]{{\BB{\mathcal{K}}^{(t)}}^{-1}}
\newcommand{\KKinvii}[0]{{\KK^{-1}}_{ii}}
\newcommand{\KKinvY}[0]{\KK^{-1} \B{y}}
\newcommand{\KKinvYi}[0]{\left(\KKinvY\right)_i}

\title{\bfseries\sffamily Detecting Label Noise via Leave-One-Out Cross-Validation}

\author[1*]{\normalsize Yu-Hang Tang}
\author[2]{\normalsize Yuanran Zhu}
\author[1]{\normalsize Wibe A. de Jong}
\affil[1]{\small Lawrence Berkeley National Laboratory, Berkeley, California 94720, USA}
\affil[2]{\small University of California, Merced, Merced, California 95340, USA}
\affil[*]{\small Correspondence: \url{tang@lbl.gov}}
\date{}

\begin{document}

\maketitle

\paragraph{Abstract}

We present a simple algorithm for identifying and correcting \emph{real-valued} noisy labels from a mixture of clean and corrupted sample points using Gaussian process regression. A heteroscedastic noise model is employed, in which additive Gaussian noise terms with independent variances are associated with each and all of the observed labels. Optimizing the noise model using maximum likelihood estimation leads to the containment of the GPR model's predictive error by the posterior standard deviation in leave-one-out cross-validation. A multiplicative update scheme is proposed for solving the maximum likelihood estimation problem under non-negative constraints. While we provide a proof of convergence for certain special cases, the multiplicative scheme has empirically demonstrated monotonic convergence behavior in virtually all our numerical experiments. We show that the presented method can pinpoint corrupted sample points and lead to better regression models when trained on synthetic and real-world scientific data sets.

\vspace{1em}

\section{Introduction}

Machine learning algorithms that can generate robust models from noisy data are of great practical importance. Here, the noisy labels are considered tempered from their unobserved ``clean'' versions either by an adversary or due to an error in the data acquisition process. Two closely related objectives are usually involved, namely 1) model and training robustness in the presence of label noise and 2) identification and quantification of the noise as feedback to the data collection mechanism.

The focus of existing work is often divided by the type of labels that are of interest. A major body of research has been carried out concerning binary and categorial labels~\cite{crammerLearningGaussianHerding2010,crammerAdaptiveRegularizationWeight2013} with a plethora of theoretical work concerning topics such as lower bounds on the sample complexity ~\cite{aslamSampleComplexityNoisetolerant1996}, sample-efficient strategies for learning binary-valued functions~\cite{cesa-bianchiSampleefficientStrategiesLearning1999},
empirical risk minimization~\cite{natarajanLearningNoisyLabels2013}, learnability of linear threshold functions~\cite{bylanderLearningLinearThreshold1994}, and the role of loss functions~\cite{bartlettConvexityClassificationRisk2006,maNormalizedLossFunctions2020}. A particular application area of interest is image classification and visual recognition.~\cite{biggioSupportVectorMachines2011,xiaoLearningMassiveNoisy2015,liLearningNoisyLabels2017,liDivideMixLearningNoisy2020,karimiDeepLearningNoisy2020,yaoSearchingExploitMemorization2020,jiangSyntheticNoiseDeep2020}. Many works targeting both robust training and noisy label identification have been proposed. For example, Wu \textit{et al.} presented an algorithm for recognizing incorrectly labeled sample points using an iterative topological filtering process~\cite{wuTopologicalFilterLearning2020}. Tanaka \textit{et al.} proposed a joint optimization framework for learning deep neural network parameters and estimating true labels for image classification problems by an alternating update of network parameters and labels~\cite{tanakaJointOptimizationFramework2018}. The strategy of sample selection and using the small-loss trick have also given rise to several algorithms and frameworks for identifying and correcting noisy classification labels~\cite{hanCoteachingRobustTraining2018,jiangMentorNetLearningDataDriven2018,songSELFIERefurbishingUnclean2019,weiCombatingNoisyLabels2020,yuHowDoesDisagreement2019}.

In this paper, our goal is to build regression models and identify corrupted sample points from a \emph{mixture} of clean and noisy \emph{continuous} labels. One motivation behind is to screen high-throughput simulation and experimental data sets, where the results may be corrupted due to random faults but are otherwise of good accuracy~\cite{spotte-smithQuantumChemicalCalculations2021}. There are comparably fewer pieces of work on this topic, despite that real-valued labels are ubiquitous and essential in data-driven modeling for scientific and engineering purposes~\cite{rasheedDigitalTwinValues2020}. The most relevant existing work concerns heteroscedastic Gaussian process regression. Goldberg \textit{et al.} treated the variance of the label noise as a latent function dependent on the input and modeled it with a second Gaussian process~\cite{goldbergRegressionInputdependentNoise1997}. Le \textit{et al.} presented an algorithm to estimate the variance of the Gaussian process locally using maximum a posterior estimation solved by Newton's method~\cite{leHeteroscedasticGaussianProcess2005}. However, both work regard the noise level as a property of the underlying Gaussian random field rather than that of individual sample points. As a consequence, it is not straightforward to single out anomaly on a per-label basis using these methods. Cesa-Bianchi \textit{et al.} proposes a general treatment on online learning of real-valued linear and kernel-based predictors. The method, however, has to use \emph{multiple copies} of each example~\cite{cesa-bianchiOnlineLearningNoisy2011}. To the best of our knowledge, simultaneous robust training and label-wise noise identification for real-valued regression remains as a challenging task.

This paper proposes a new algorithm that can identify noisy labels while generating accurate GPR models using data sets of high noise rates.  The method detects noisy labels using a label-wise heteroscedastic noise model, which leads to a Tikhonov regularization term with independent values for each sample point. While various forms of regularization have been widely used in GPR to control overfitting, our work is the first to relate heteroscedastic regularization with noisy label identification. Optimizing the noise model using maximum likelihood estimation would ensure that the GPR leave-one-out cross-validation errors on the training set are bounds by the predictive uncertainties. A simple multiplicative update scheme, which is monotonically converging for optimizing the noise model, allows quick implementation and fast computation.

\section{Preliminaries}
\label{sec:preliminaries}

\paragraph{Notations} Upper case letters, \textit{e.g.} $M$, denote matrices. Bold lower case letters, \textit{e.g.} $\B{a}$, denote column vectors. Regular lower case letters, \textit{e.g.} $x$, denote scalars. Vectors are assumed to be column vectors by default. $\diag{\B{a}}$ denotes a diagonal matrix whose diagonal elements are specified by $\B{a}$. $\DIAG{A}$ denotes a vector formed by the diagonal elements of $A$. $\B{1}^i$ denotes the $i$-th canonical base vector, \textit{i.e.} an `indicator' vector whose $i$-th element is 1 and all other elements are 0. $\odot$ is the element-wise product operation.

\paragraph{Gaussian process regression} Given a dataset $\mathcal{D}$ of $N$ samples points $\{(\B{x}_i, y_i); \B{x}_i \in \BB{\Phi}, y_i \in \mathbb{R}, i = 1, \ldots, N\}$ and a covariance function, interchangeably called a kernel, $\kappa: \BB{\Phi} \times \BB{\Phi} \mapsto \mathbb{R}$, a Gaussian process regression (GPR) model \cite{Rasmussen2006a} can be learned by treating the labels as instantiations of random variables from a Gaussian random field.

For an unknown sample point at location $\B{x}_*$, the prediction for its label $y_*$ made by a GPR model takes the form of a posterior normal distribution with mean $\mu_* = \B{k}_* K^{-1} \B{y}$ and variance $\sigma_*^2 = \kappa(\B{x}_*, \B{x}_*) - \B{k}_*^\T K^{-1} \B{k}_*$. Here, $K_{ij} \doteq \kappa(\B{x}_i, \B{x}_j)$ is the pairwise matrix of prior covariance between the sample points as computed by the covariance function; $({\B{k}_*})_i \doteq \kappa(\B{x}_*, \B{x}_i)$ is the vector of prior covariance between $\B{x}_*$ and the training sample; $\B{y} \doteq \left[y_1, \ldots, y_N \right]^\T$ is a vector containing all training labels. Note that we have made the assumption, without loss of generality, that the prior means of the labels are all zeros.

The likelihood of a data set $\mathcal{D}$ given a GPR model, which describes how well the model can explain and fit the observed data, usually takes the form of a multivariate normal:
\begin{equation}
    p(\mathcal{D} \mid \BB{\theta}) = (2\pi)^\frac{N}{2} \left|K\right|^{-\frac{1}{2}} \exp\left[ -\frac{1}{2} \B{y}^\T K^{-1} \B{y} \right],\label{eq:GP-likelihood}
\end{equation}
where $\BB{\theta}$ is a vector containing all the model hyperparameters. Training a GPR model often involves maximum likelihood estimation of $\BB{\theta}$ which attemps to find $\BB{\theta}^* = \underset{\BB{\theta}}{\mathrm{argmax}}\ p(\mathcal{D} \mid \BB{\theta})$.

\section{Noisy Label Detection}
\label{sec:anomaly-detection}

\subsection{For a Single Label}
\label{sec:a-motivating-example}

Consider the toy problem where a single sample point $(\B{x}_k, y_k)$ from a dataset $\mathcal{D}$ might contain label noise. \Cref{fig:motivating-example} illustrates the situation using three versions of a data set with twelve clean labels and one potentially noisy label. To figure out to what degree we could trust $y_k$ in absence of the ground truth, an intuition is to examine the posterior likelihood of the suspicious point given a GPR model trained only on the clean data $\mathcal{D}\setminus\{ \B{x}_k \}$. Using the clean-label GPR posterior mean and standard eviation $(\mu_{-k}, \sigma_{-k}) \doteq (\mu_{k}, \sigma_k) \mid \mathcal{D} \setminus \{ \B{x}_k \}$, we can infer the trustworthiness of the label by examining whether an estimate for the magnitude of the noise, which widens the posterior confidence interval, needs to be added to the prior covariance matrix $K$ to ensure $y_k \in [\mu_{-k} - \sigma_{-k}, \mu_{-k} + \sigma_{-k}]$. 

\begin{figure*}
    \includegraphics[width=\textwidth]{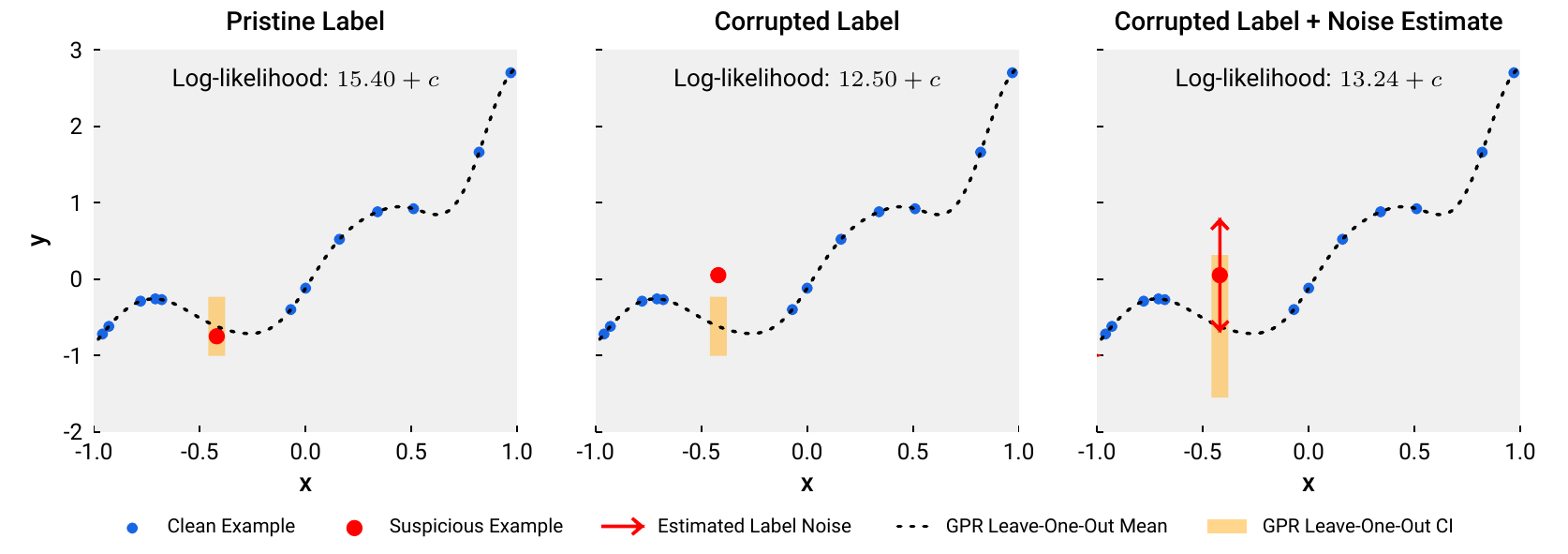}
    \caption{Given a dataset of many clean labels and one potentially noisy label, we could use a GPR model trained on only the clean sample points, \textit{i.e.} a leave-one-out model, to infer the noisiness of the suspicious label. In the left panel, the label of question is indeed clean and is bound by the 1$\sigma$ confidence interval of the leave-one-out model. In the middle panel, the label does contain noise and hence breaks away from the leave-one-out confidence interval. In the right panel, we attempt to bound the noisy label again by the confidence interval by adding an estimate of the noise magnitude into the prior covariance of the sample point. Such estimate results in a higher likelihood of the data and thus confirms that the label contains noise.\label{fig:motivating-example}}
\end{figure*}

The example in \cref{fig:motivating-example} is closely related to leave-one-out cross-validation (LOOCV) using GPR. In fact, the difference between $y_k$ and $\mu_{-k}$ is the leave-one-out cross-validation error at $\B{x}_k$ of a GPR model trained on the entire dataset, for which a closed form solution exists as $e_{-k} = \frac{(K^{-1} \B{y})_k}{(K^{-1})_{kk}}$. The width of the 1$\sigma$ confidence interval also has a closed-form expression as $\sigma_{-k} = \sqrt{\frac{1}{(K^{-1})_{kk}}}$. The determination of whether $y_k$ contains noise is based on the belief that a clean label would require a zero or small prior noise prescription to ensure $e_{-k} < \sigma_{-k}$ .

\subsection{For All Labels}
\label{sec:loocv-and-tikhonov}

A difficulty of directly applying the decision process as depicted in \Cref{fig:motivating-example} to a data set where many labels are noisy is that it is unclear which sample points can be trusted to bootstrap the cross validation process. Instead, the identification of the noisy labels has to be learned via an iterative process.

Now we formalize the algorithm for noisy label detection using leave-one-out cross-validation with GPR. We assume that each observed label $y_i$ is a combination of the ground truth label $f(\B{x}_i)$ with an additive Gaussian noise $\varepsilon_i \sim \mathcal{N}(0, \sigma_i)$ of zero mean and variance $\sigma_i$, \textit{i.e.} $y_i = f(\B{x}_i) + \varepsilon_i$. Further assuming that the noise terms are mutually independent and also independent of the true labels, \textit{i.e.} $\Cov{\varepsilon_i}{\varepsilon_j} = \delta_{ij} \sigma_i \sigma_j$ and $\Cov{y_i}{\varepsilon_j} \equiv 0\ \forall i, j$, the GPR covariance matrix of this training set then becomes
\begin{align}
    \KK = K + \Sigma,
\end{align}
where $K$ is the prior covariance matrix as computed by the kernel, while $\Sigma = \diag{\BB{\sigma}} \doteq \diag{\left[\sigma_1, \ldots, \sigma_N \right]^\T}$ is a diagonal covariance matrix between the noise terms. Note that $\Sigma$ can be regarded as a generalization of the basic Tikhonov regularization term where $\Sigma = \sigma \mathbb{I}$ is uniform across the sample.

To determine $\BB{\sigma}$, which quantifies the error each label contains, we  maximize the likelihood of the dataset with respect to the regularized kernel matrix $\KK$. By plugging $\KK$ into \Cref{eq:GP-likelihood}, we obtain the following negative log-likelihood function:
\begin{equation}
\begin{aligned}
    \LL(\B{y} \mid \B{X}, \kappa, \Sigma)
    & \doteq \log \Det{\KK}
    +\B{y}^\T \KK^{-1} \B{y}
    + c. \\
    & \propto -\log p(\B{y} \mid \B{X}, \kappa, \Sigma)
    \label{eq:negative-log-likelihood}
\end{aligned}
\end{equation}
This leads to the following constrained optimization problem:
\begin{equation}\label{eq:optimization-problem}
\begin{aligned}
    \underset{\BB{\sigma}, \BB{\theta}}{\mathrm{argmin}}\ & \LL(\B{y} \mid \B{X}, \kappa, \Sigma)\\
    \text{subject to } & \sigma_i \ge 0\ \forall\ i.
\end{aligned}
\end{equation}

Following the derivation in \Cref{section:gradient-of-negative-log-likelihood-function}, the gradient of \eqref{eq:negative-log-likelihood} with respect to $\BB{\sigma}$ is
\begin{equation}
    \frac{\D \LL}{\D\BB{\sigma}} = \diag{\KK^{-1}} - \left(\KKinvY\right) \odot \left(\KKinvY\right). \label{eq:dL-dSigma-diagonal}
\end{equation}
Thus, a necessary condition for $\LL$ to reach an optimum subject to the non-negativity constraint on $\BB{\sigma}$ is
\begin{equation}
    \KKinvii - \KKinvYi^2 \ge 0\ \forall\ i.\label{eq:optimality-condition}
\end{equation}
Since $\KKinvii > 0$ due to the fact that $\KK$ and $\kappa$ are positive definite, condition \eqref{eq:optimality-condition} can be rearranged as
\begin{equation}
     \left[\frac{\KKinvYi}{\KKinvii}\right]^2 \le \frac{1}{\KKinvii}\ \forall\ i
    \label{eq:optimality-condition-altform}
\end{equation}

Recall that $\tfrac{\KKinvYi}{\KKinvii}$ and $\tfrac{1}{\KKinvii}$ are the leave-one-out cross-validation error and posterior variance of the regularized GPR model for sample point $i$, respectively. What \cref{eq:optimality-condition-altform} conveys is that
\begin{equation*}
\begin{aligned}
    &\text{LOOCV error $\le$ LOOCV standard deviation $\forall\ i$}\\
    &\text{when $\BB{\sigma}, \BB{\theta}$ optimizes $\LL$ subject to $\BB{\sigma} \ge 0$.}
\end{aligned}
\end{equation*}
In other words, the maximum likelihood estimation of the kernel and the noise model seeks to bound the the leave-one-out cross validation errors on the training set by the predictive uncertainty, thus minimizing the surprise incurred by any inconsistency between the labels and the prior covariance matrix.

\begin{figure*}[ht!]
    \includegraphics[width=\textwidth]{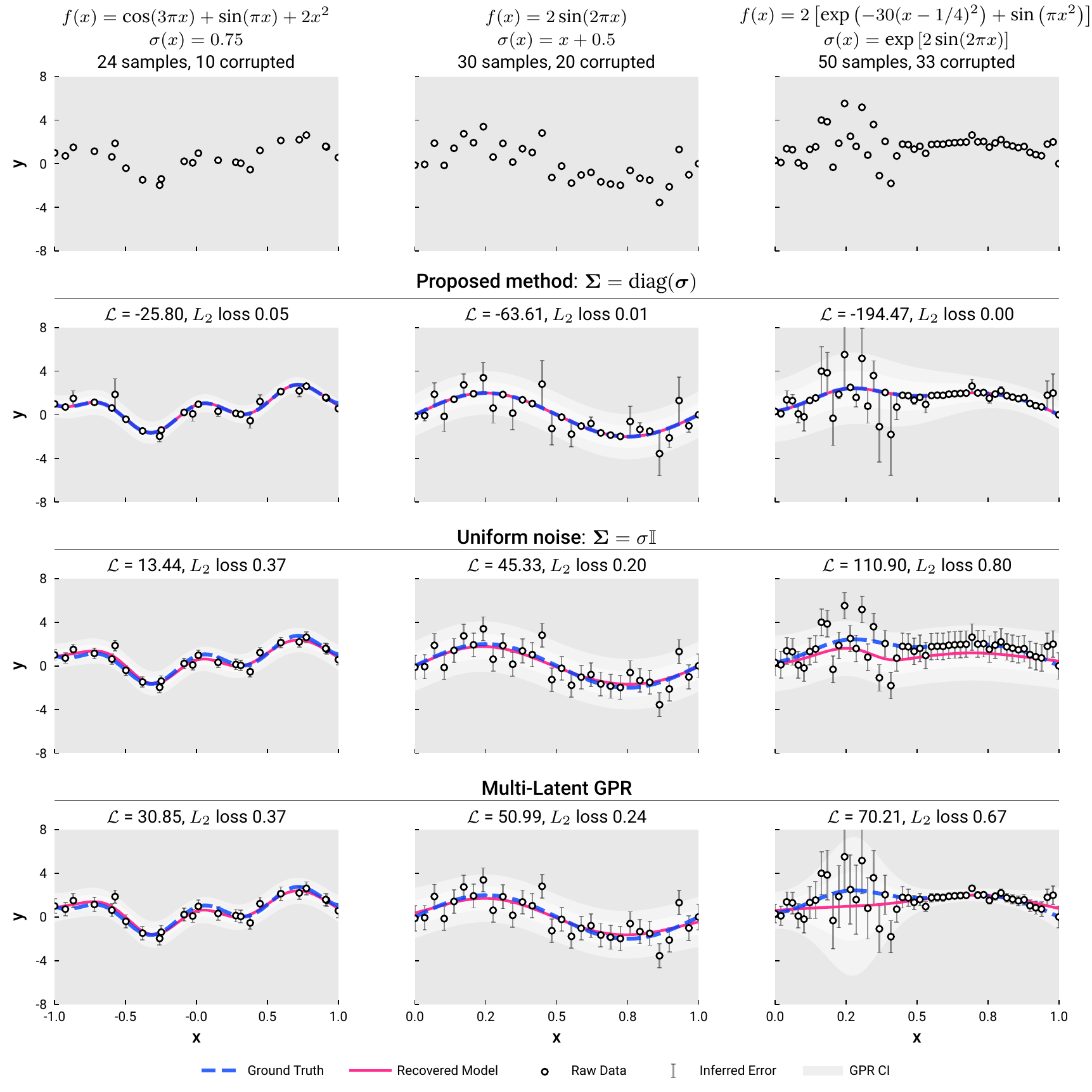}
    \caption{A comparison between the algorithm proposed in \Cref{sec:loocv-and-tikhonov}, the uniform noise model, and a multi-latent heteroscedastic GPR implementation~\cite{deg.matthewsGPflowGaussianProcess2017}. Three sets of 1D functions and noise distributions are tested for illustrative purposes.\label{fig:1D-demo}}
\end{figure*}

In \Cref{fig:1D-demo}, we compare the proposed method against the uniform noise model and a multi-latent GPR model~\cite{goldbergRegressionInputdependentNoise1997,lazaro-gredillaVariationalHeteroscedasticGaussian2011,saulChainedGaussianProcesses2016} using three synthetic examples. In the first example, the ground truth function $f(x) = \cos(3 \pi x) + \sin(\pi x) + 2 x^2$ are measured at 24 points drawn from a uniform grid between $[-1, 1]$ perturbed with i.i.d. noises $\sim \mathcal{N}(0, 0.05)$, while 10 of the 24 labels are contaminated with i.i.d. noises $\sim \mathcal{N}(0, 0.75)$. In the second example, we use the setup in Ref.~\cite{goldbergRegressionInputdependentNoise1997} with a reduced sample size of 30 while contaminating 20 of the labels. In the third example, we use the setup in Ref.~\cite{leHeteroscedasticGaussianProcess2005} with a reduced sample size of 50 while contaminating 33 of the labels. The kernel hyperparameters and the noise parameters in all methods are optimized by multiple maximum likelihood estimation runs from randomized initial guesses. The proposed method consistently delivers better performance in terms of identifying noisy labels and learning accurate regression models.

\section{Solution Techniques}

\subsection{A Multiplicative Update Scheme}

It is challenging to derive closed-form solutions for problem \eqref{eq:optimization-problem} due to the nonconvexity and nonlinearity of the loss function. While solving problem \eqref{eq:optimization-problem} using a gradient-based optimizer might be convenient, especially for the joint optimization of $\BB{\sigma}$ and $\BB{\theta}$ by concatenating $\tfrac{\D \LL}{\D \BB{\sigma}}$ \eqref{eq:dL-dSigma-diagonal} with $\tfrac{\D \LL}{\D \BB{\theta}} = \mathrm{tr}\left( \KKinv \tfrac{\D K}{\D \BB{\theta}} \right) - \left( \KKinvY \right)^\T \tfrac{\D K}{\D \BB{\theta}} \left( \KKinvY \right)$, the rate of convergence could be slow as demonstrated in \Cref{fig:optim-algs-comparison}. Meanwhile, The problem size, which scales linearly with the number of sample points, practically prevents the usage of more sophisticated algorithms such as L-BFGS-B.

However, observe that both terms in the right hand side of \eqref{eq:dL-dSigma-diagonal} possess only non-negative elements. More precisely, 
each element of $\diag{\KKinv}$ is positive as a consequence of the positive-definiteness of $\KK$, while each element of $\left(\KKinvY\right) \odot \left(\KKinvY\right)$ is non-negative due to the element-wise square operation.
This prompts the following multiplicative update scheme for optimizing $\LL$:
\begin{equation}
    \sigma^{(t+1)}_i = \sigma^{(t)}_i \cdot \frac{\left(\KKtinv\B{y}\right)_i^2}{\diag{\KKtinv}_i}, \label{eq:multiplicative-update-rule}
\end{equation}
where the superscript $^{(t)}$ indicates the version of $\BB{\sigma}$ at iteration step $t$. The non-negativity constraint on $\BB{\sigma}$ is automatically honored as long as the starting point $\BB{\sigma}^{(0)}$ is non-negative.

Meanwhile, the multiplicative factor is greater than one and increases $\sigma_i$ after the update when $\KKinvYi^2 > \KKinvii$, \textit{i.e.} when the slope along $\sigma_i$ is negative.
Similarly, the multiplicative factor is less than one and decreases $\sigma_i$ when the slope is positive.
This aligns well with the intuition that the optimization should follow the gradient, whose sign is determined by the relative magnitude of the elements of $\KKinvYi^2$ and $\KKinvii$.

Surprisingly, the multiplicative update scheme has exhibited monotonic convergence behavior in virtually all synthetic and real-world cases that we have tested. The rate of convergence is very fast as shown in \Cref{fig:optim-algs-comparison}, in which we compare the scheme against several gradient descent and quasi-Newton methods.

\begin{figure*}[ht!]
    \includegraphics[width=\textwidth]{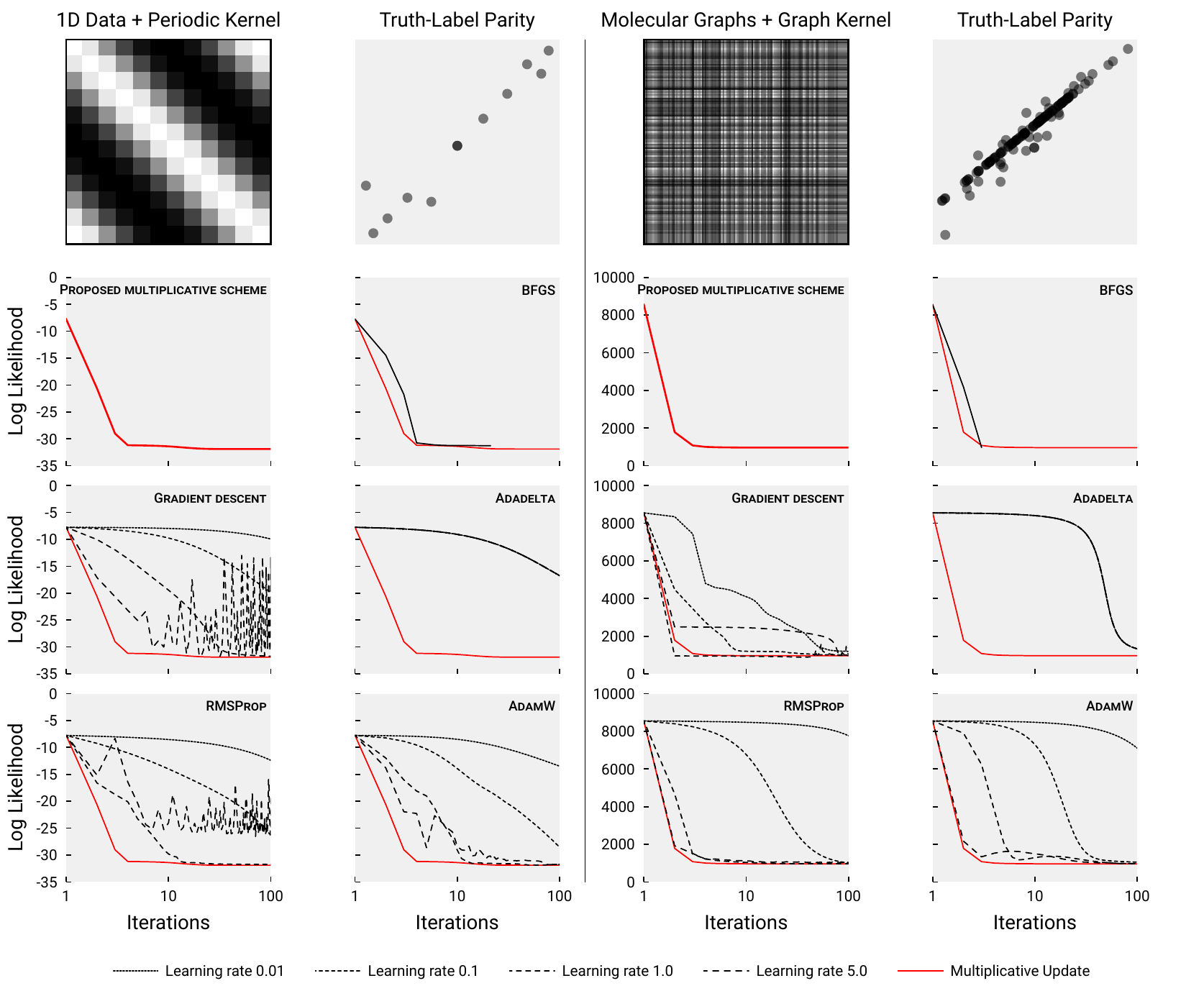}
    \caption{An empirical comparison of the multiplicative update scheme as in \cref{eq:multiplicative-update-rule} against the quasi-Newton method L-BFGS-B and four variants of gradient descent algorithms. The multiplicative update scheme outperforms virtually all other algorithms for optimizing $\BB{\sigma}$ in terms of the rate of convergence and the number of function evaluations.\label{fig:optim-algs-comparison}}
\end{figure*}

The fact that the multiplicative update scheme for optimizing $\BB{\sigma}$ can converge monotonically in a very efficient manner creates many opportunities for the joint optimization of $\BB{\theta}$ and $\BB{\sigma}$. One possibility is to cast the problem into the bi-level optimization paradigm, where $\BB{\theta}$ is treated by an upper-level optimizer, while the multiplicative update of $\BB{\sigma}$ as the lower-level optimizer. Another possibility is to interleave the optimization of $\BB{\sigma}$ and $\BB{\theta}$ in alternating steps, thus forming a block coordinate descent scheme. 

\REM{block coordinate descent?}

One last note is that the multiplicative update scheme reduces to $\sigma^{(t+1)} = \sigma^{(t)} \cdot \frac{\left(\KKtinv\B{y}\right)^\T \left(\KKtinv\B{y}\right)}{\mathrm{Tr}\left[{\KKtinv}\right]}$ when $\Sigma = \sigma \mathbb{I}$. This is also empirically observed to be monotonically converging and hence could be used to efficiently tune the Tikhonov regularization term in standard GPR models.

\subsection{Preliminary Convergence Analysis}

\REM{What if likelihood is unimodal? What if it is not unimodal due to the non-negative constraint?}

It is easy to verify that the stationary points of $\mathcal{L}$ are the fixed points of the multiplicative update rule. To see that, note that $\frac{\D\LL}{\D \BB{\sigma}} = \BB{0}$ implies $\diag{\KKinv}_i = \KKinvYi^2\ \forall\ i$. In that case, we have all the update coefficients $\frac{\diag{\KKinv}_i}{\KKinvYi^2} = 1$, which indicates that the recurrence relationship has reached a fixed point. Also note that the zero vector $\BB{0}$ is a trivial fixed point of the multiplicative update rule.

In this section, we attempt to provide some theoretical insight into the observed monotonic convergence behavior of the multiplicative update scheme. However, a rigorous proof for the general case is not available yet.

\begin{theorem}
For a special case where $K=\diag{\B{k}}$ is a diagonal matrix, if the optimization problem \eqref{eq:optimization-problem} has solutions, then the solution is unique and the multiplicative update scheme \eqref{eq:multiplicative-update-rule} monotonically converges to the unique solution.
\end{theorem}
\begin{proof}
We use the fixed-point theorem to get the result. When $K$ is diagonal, $\frac{\D\LL}{\D \BB{\sigma}} = \BB{0}$ can be solved independently for each $i$ which yields the exact result:
\begin{align*}
 \frac{y_i^2}{(K_{ii}+\sigma^{*}_i)^2}=\frac{1}{(K_{ii}+\sigma^{*}_i)^2}
 \qquad \Rightarrow\qquad 
\sigma^{*}_i=y_i^2-K_{ii}.
\end{align*}
Since $K=\diag{\B{k}}$ is a symmetric, positive definite matrix, the constraint $\sigma^{*}_i\geq 0$ implies $0< K_{ii}/y_i^2\leq 1$. For such case, iteration scheme \eqref{eq:multiplicative-update-rule} reduces to
\begin{align*}
    \sigma^{(t+1)}_i = \sigma^{(t)}_i\frac{y_i^2}{\sigma^{(t)}_i+K_{ii}}, \qquad 1\leq i\leq N,
\end{align*}
which can be reformulated as 
\begin{align*}
    \frac{1}{\sigma^{(t+1)}_i} = \frac{1}{y_i^2}+\frac{K_{ii}}{y_i^2}\frac{1}{\sigma^{(t)}_i}=g_i\left(\frac{1}{\sigma^{(t)}_i}\right).
\end{align*}
Set $p^{(t+1)}_{i}=1/\sigma^{(t+1)}_i$ and set $p_i^*=1/\BB{\sigma}_i^*$ to be the exact fixed point. Now, the mean-value theorem implies the $(t+1)$-step error can be bounded as
\begin{align*}
     \left|p^{(t+1)}_{i}-p^{*}_{i}\right|
    &=\left|g\left(p^{(t)}_{i}\right)-g\left(p^{*}_{i}\right)\right| \\
    &=\left|g'\left(\xi^{(t+1)}\right)\right|\left|p^{(t)}_{i}-p^{*}_{i}\right| \\
    &=\frac{K_{ii}}{y_i^2}\left|p^{(t)}_{i}-p^{*}_{i}\right|.
\end{align*}
Given the prerequisite $0<K_{ii}/y_i^2\leq 1$, when $K_{ii}/y_i^2=k_i<1$, we can get the monotonic convergence result:
\begin{align}\label{sigma_1neq0}
    \left|p^{(t+1)}_{i}-p^{*}_{i}\right|\leq k_i^n \left|p^{(0)}_{i}-p^{*}_{i}\right| \rightarrow 0 \qquad \text{as}\qquad n\rightarrow \infty.
\end{align}
If otherwise $K_{ii}/y_i^2=k_i=1$, we have 
\begin{align}\label{sigma_1=0}
    p^{(t+1)}_{i}=\frac{n}{y_i^2}+p^{(0)}_{i} \rightarrow +\infty \qquad \text{as}\qquad n\rightarrow \infty,
\end{align}
which corresponds to a monotonically convergent $\sigma^{(t)}_i\rightarrow 0$ since $\sigma^{(t)}_i=1/p^{(t)}_{i}$. Hence, for this special case, we conclude that if there exists a solution to \eqref{eq:optimization-problem}, then the solution is unique and the multiplicative update scheme \eqref{eq:multiplicative-update-rule} will monotonically converge to this unique solution.
\end{proof}

Immediately, we can get the convergence of the loss function as below.
\begin{corollary}
For the special case where $K=\diag{\B{k}}$ is a diagonal matrix, if the optimization problem \eqref{eq:optimization-problem} has solutions, then the multiplicative update scheme  \eqref{eq:multiplicative-update-rule} generates a convergent sequence $\LL(\BB{\sigma}^{(t)})\rightarrow \LL(\BB{\sigma}^{*})$, where $\LL(\BB{\sigma}^{*})$ is the optimal loss. If the exact solution $\sigma^{*}_i>0$ for all $1\leq i\leq N$, then the convergence rate is $R$-linear, otherwise, it is $R$-superlinear.
\end{corollary}
\begin{proof}
When $K=\diag{\B{k}}$ is a diagonal matrix, using the definition of $\LL$ \eqref{eq:negative-log-likelihood}, we can get the loss at the $t$-th iteration: 
\begin{align}\label{eqn:L}
\LL(\BB{\sigma}^{(t)})\propto \sum_{i=1}^N\log\left(\sigma_{i}^{(t)}+K_{ii}\right)+\sum_{i=1}^N\frac{y_i^2}{\sigma_{i}^{(t)}+K_{ii}}
\end{align}
To get the $R$-convergence of $\LL$, we only need to prove that the upper bound of $\left|\LL(\BB{\sigma}^{(t)})-\LL(\BB{\sigma}^{*})\right|$ converges to 0. Consider the $i$-th term in the first summation in \eqref{eqn:L}, we have the error estimate: 
\begin{align}\label{eqn:L1}
\left|\log\left(\sigma_{i}^{(t)}+K_{ii}\right)-\log\left(\sigma_{i}^{*}+K_{ii}\right)\right|\leq \frac{1}{K_{ii}}\left|\sigma_{i}^{(t)}-\sigma_{i}^{*}\right|.
\end{align}
Here we used the Lipschitz condition of $\log(x)$ on domain $[K_{ii},+\infty)$, i.e.  $\left|\log(x)-\log(y)\right|\leq\left|x-y\right|/K_{ii}$. For the $i$-th term in the second summation in \eqref{eqn:L}, we also have 
\begin{align}\label{eqn:L2}
& \left|\frac{y_i^2}{\sigma_{i}^{(t)}+K_{ii}}-\frac{y_i^2}{\sigma_{i}^{*}+K_{ii}}\right| \\
\leq & \left|\frac{y_i^2}{\left(\sigma_{i}^{(t)}+K_{ii}\right)\left(\sigma_{i}^{*}+K_{ii}\right)}\right|\left|\sigma_{i}^{(t)}-\sigma_{i}^{*}\right| \\
\leq & \frac{y_i^2}{K_{ii}\left(\sigma_{i}^{*}+K_{ii}\right)}\left|\sigma_{i}^{(t)}-\sigma_{i}^{*}\right|
\end{align}
Combining estimate \eqref{eqn:L1}-\eqref{eqn:L2}, we can find a constant $C$ such that 
\begin{align}\label{eqn:L3}
\left|\LL(\BB{\sigma}^{(t)})-\LL(\BB{\sigma}^{*})\right|\leq C
\left\|\BB{\sigma}^{(t)}-\BB{\sigma}^{*}\right\|_{\infty}.
\end{align}
The above upper bound converges to 0 linearly if the exact solution $\sigma^{*}_i>0$ for all $1\leq i\leq N$. Otherwise, we can only get the superlinear convergence according to \eqref{sigma_1=0} because the convergence rate is of the order $\mathcal{O}(1/n)$ in the $j$-th direction where $\sigma^{*}_j=0$.
\end{proof}

\REM{Convexity of the log likelihood when restricted to a neighborhood of each optimum?}

\REM{Why the Multiplicative rule give rise to monotonic convergence?}

\REM{Jacobian of fixed-point iteration:

\Cref{eq:multiplicative-update-rule} can be regarded as a fixed-point iteration
\begin{equation}
\begin{aligned}
    \BB{\sigma}_i = g(\BB{\sigma})_i = \BB{\sigma_i} f(\BB{\sigma}_i)
\end{aligned}
\end{equation}
where
\begin{equation}
\begin{aligned}
    f(\BB{\sigma})_i = \frac{\KKinvYi^2}{\KKinvii}.
\end{aligned}
\end{equation}

Function $g(\cdot)$ has the Jacobian
\begin{equation}
\begin{aligned}
\frac{\partial g(\BB{\sigma})_i}{\partial \BB{\sigma}_j}
    &= 
    \delta_{ij} f(\BB{\sigma})_j +
    \BB{\sigma}_i \left[
        -2 \frac{\mathbf{a}_i \mathbf{a}_j (\KKinv)_{ij}}{\mathbf{b}_i}
        + \frac{\mathbf{a}_i^2 (\KKinv)_{ij}^2}{\mathbf{b}_i^2}
    \right] \\
    &= \delta_{ij} \frac{\KKinvYi^2}{\KKinvii} +
    \BB{\sigma}_i \left[
        -2 \frac{(\KKinv \mathbf{y})_i (\KKinv \B{y})_j (\KKinv)_{ij}}{\KKinvii}
        + \frac{\KKinvYi^2 (\KKinv)_{ij}^2}{(\KKinvii)^2}
    \right],
\end{aligned}
\end{equation}
where $\B{a}_i \doteq \KKinvYi$ and $\B{b}_i \doteq \KKinvii$.
}

\REM{Scalar case as continued fraction}

\subsection{Penalty and Sparsification}

If a sparse solution is desired or if we want to prevent the method from being too aggressive in calling out noisy labels, an $\lp$ penalty term can be introduced into the loss function:
\begin{equation}
\begin{aligned}
    \LL^{p}_{\lambda}
    &= \LL + \lambda \lVert \BB{\sigma} \rVert_p^p.
\end{aligned}
\end{equation}

Due to the non-negative nature of the gradient of the $\lp$ term, it can also be incorporated into the multiplicative update scheme as
\begin{equation}
    \sigma^{(t+1)}_i = \sigma^{(t)}_i \cdot \frac{\left(\KKtinv\B{y}\right)_i^2}{\diag{\KKtinv}_i + \lambda\,p \left| \sigma_i^{(t)} \right|^{p-1}}. \label{eq:multiplicative-update-rule1}
\end{equation}

\REM{While the method that we proposed has been described in a context of noisy label learning, the problem is closely related to anomaly detection. Common outlier and anomaly detection algorithms use sampling-based approaches such as MCMC. Here, a sample point is anomalous if it deviates from the LOOCV mean by too much.}

\REM{Talk about the noise model also quantifies the portion of the variance of the dataset that is unexplainable by the model.}

\section{Experiments}

\begin{table*}[htbp!]
\caption{Label Noise - Rate/Level: the percentage of corrupted labels and the ratio between the noise and the standard deviation of the pristine labels; R\textsuperscript{2}: the coefficient of determination between the inferred and actual label noise; AUC: area under the ROC curve of a `noisy label' classifier that thresholds the learned $\sigma_i$; Precision at Recall Level: precision of the classifier at specified recall levels. Regression accuracy - plain/basic/full: $\BB{\Sigma} = 0,\ \sigma \mathbb{I},\ \diag{\BB{\sigma}}$, respectively.\label{table:qm7}}
\vspace{0.5em}
\small
\setlength{\tabcolsep}{2.0pt}
\begin{tabular}{@{\extracolsep{3pt}}llllllrrr}
\multicolumn{9}{c}{QM7 Atomization Energy} \\
\toprule
\multicolumn{2}{c}{\thead[tc]{Label Noise}} & \multicolumn{4}{c}{\thead[tc]{Noisy Label Detection\\{\scriptsize\mdseries Thresholding on $\sigma_i$}}} & \multicolumn{3}{c}{\thead[tc]{GPR Accuracy\\{\scriptsize\mdseries 5-fold CV MAE (kcal/mol)}}}\\ \cline{1-2} \cline{3-6} \cline{7-9}
\multirow{2}{*}{\thead{Rate}} & \multirow{2}{*}{\thead{Level}} & \multirow{2}{*}{\thead{R\textsuperscript{2}}} & \multirow{2}{*}{\thead{AUC}} & \multicolumn{2}{c}{\thead[tc]{Precision\\{\mdseries\scriptsize at Recall Level}}} & \multicolumn{3}{c}{\thead[tc]{\mdseries \scriptsize (Pristine Data: 1.05)}}\\ \cline{5-6}
& & & & \thead[tl]{70\%} & \thead[tl]{95\%} & \thead[tl]{Plain} & \thead[tl]{Basic} & \thead[tl]{Full} \\
\midrule
\multirow{4}{*}{10\%} & 10\% & 0.59 & .952 &   76.8\% &  27.9\% &      3.75 &      2.26 &      1.14 \\
    & 50\% & 0.98 & .991 &   96.6\% &  79.7\% &     17.78 &      4.93 &      1.12 \\
    & 100\% & 1.00 & .996 &   97.3\% &  83.7\% &     34.59 &      7.05 &      1.12 \\
    & 200\% & 1.00 & .997 &   98.6\% &  96.9\% &     68.83 &     10.21 &      1.12 \\
\cline{1-9}
\multirow{4}{*}{30\%} & 10\% & 0.86 & .949 &   92.4\% &  57.8\% &      6.55 &      2.85 &      1.20 \\
    & 50\% & 0.99 & .988 &   99.3\% &  92.3\% &     31.89 &      6.68 &      1.24 \\
    & 100\% & 1.00 & .994 &   99.3\% &  95.3\% &     64.18 &      9.83 &      1.23 \\
    & 200\% & 1.00 & .997 &   99.5\% &  98.9\% &    128.08 &     15.76 &      1.25 \\
\cline{1-9}
\multirow{4}{*}{50\%} & 10\% & 0.92 & .944 &   95.7\% &  77.7\% &      8.32 &      3.19 &      1.34 \\
    & 50\% & 1.00 & .987 &   99.6\% &  96.5\% &     42.86 &      7.13 &      1.38 \\
    & 100\% & 1.00 & .993 &   99.7\% &  97.6\% &     82.79 &     12.49 &      1.40 \\
    & 200\% & 1.00 & .995 &   99.8\% &  99.5\% &    167.64 &     16.23 &      1.40 \\
\cline{1-9}
\multirow{4}{*}{70\%} & 10\% & 0.89 & .912 &   96.7\% &  81.1\% &     10.10 &      3.57 &      1.90 \\
    & 50\% & 1.00 & .984 &   99.8\% &  98.0\% &     49.01 &      8.12 &      1.78 \\
    & 100\% & 1.00 & .992 &   99.9\% &  98.9\% &     99.44 &     12.00 &      1.70 \\
    & 200\% & 1.00 & .996 &   99.9\% &  99.7\% &    198.57 &     17.31 &      1.72 \\
\cline{1-9}
\multirow{4}{*}{90\%} & 10\% & 0.81 & .818 &   97.3\% &  92.0\% &     11.74 &      3.76 &      3.94 \\
    & 50\% & 0.99 & .961 &   99.9\% &  98.1\% &     57.50 &      8.59 &      4.07 \\
    & 100\% & 1.00 & .982 &   99.9\% &  99.3\% &    112.77 &     13.20 &      3.10 \\
    & 200\% & 1.00 & .992 &  100.0\% &  99.9\% &    221.00 &     19.08 &      2.90 \\
\bottomrule
\end{tabular}
\hfill
\small
\setlength{\tabcolsep}{2.0pt}
\begin{tabular}{@{\extracolsep{3pt}}llllllrrr}
\multicolumn{9}{c}{QM9 8K Subset Atomization Energy} \\
\toprule
\multicolumn{2}{c}{\thead[tc]{Label Noise}} & \multicolumn{4}{c}{\thead[tc]{Noisy Label Detection\\{\scriptsize\mdseries Thresholding on $\sigma_i$}}} & \multicolumn{3}{c}{\thead[tc]{GPR Accuracy\\{\scriptsize\mdseries 5-fold CV MAE (kcal/mol)}}}\\ \cline{1-2} \cline{3-6} \cline{7-9}
\multirow{2}{*}{\thead{Rate}} & \multirow{2}{*}{\thead{Level}} & \multirow{2}{*}{\thead{R\textsuperscript{2}}} & \multirow{2}{*}{\thead{AUC}} & \multicolumn{2}{c}{\thead[tc]{Precision\\{\mdseries\scriptsize at Recall Level}}} & \multicolumn{3}{c}{\thead[tc]{\mdseries \scriptsize (Pristine Data: 1.45)}}\\ \cline{5-6}
& & & & \thead[tl]{70\%} & \thead[tl]{95\%} & \thead[tl]{Plain} & \thead[tl]{Basic} & \thead[tl]{Full} \\
\midrule
\multirow{4}{*}{10\%} & 10\% & 0.69 & .952 &  72.4\% &  32.5\% &      5.61 &      2.98 &      1.54 \\
    & 50\% & 0.99 & .983 &  91.0\% &  72.9\% &     26.02 &      5.80 &      1.58 \\
    & 100\% & 1.00 & .993 &  94.1\% &  84.3\% &     52.52 &      7.99 &      1.55 \\
    & 200\% & 1.00 & .994 &  96.7\% &  89.7\% &    104.55 &     10.74 &      1.55 \\
\cline{1-9}
\multirow{4}{*}{30\%} & 10\% & 0.86 & .945 &  90.2\% &  60.5\% &      8.99 &      3.63 &      1.71 \\
    & 50\% & 0.99 & .986 &  97.5\% &  90.9\% &     46.06 &      7.48 &      1.72 \\
    & 100\% & 1.00 & .989 &  98.1\% &  93.5\% &     89.46 &     11.03 &      1.72 \\
    & 200\% & 1.00 & .995 &  99.0\% &  97.3\% &    181.18 &     15.78 &      1.70 \\
\cline{1-9}
\multirow{4}{*}{50\%} & 10\% & 0.91 & .937 &  94.8\% &  75.4\% &     12.09 &      4.14 &      1.91 \\
    & 50\% & 0.99 & .982 &  98.9\% &  95.3\% &     55.84 &      7.73 &      1.94 \\
    & 100\% & 1.00 & .988 &  99.1\% &  96.7\% &    117.07 &     10.17 &      1.95 \\
    & 200\% & 1.00 & .995 &  99.7\% &  99.1\% &    239.49 &     16.50 &      1.91 \\
\cline{1-9}
\multirow{4}{*}{70\%} & 10\% & 0.89 & .897 &  96.0\% &  79.0\% &     14.19 &      4.34 &      2.59 \\
    & 50\% & 0.99 & .979 &  99.5\% &  97.4\% &     71.74 &      8.88 &      2.39 \\
    & 100\% & 1.00 & .987 &  99.6\% &  98.4\% &    137.61 &     12.37 &      2.28 \\
    & 200\% & 1.00 & .994 &  99.8\% &  99.5\% &    281.35 &     19.02 &      2.31 \\
\cline{1-9}
\multirow{4}{*}{90\%} & 10\% & 0.80 & .809 &  97.2\% &  91.3\% &     16.01 &      4.87 &      5.07 \\
    & 50\% & 0.99 & .955 &  99.9\% &  97.7\% &     80.58 &      9.39 &      5.18 \\
    & 100\% & 1.00 & .971 &  99.8\% &  98.5\% &    157.74 &     14.64 &      4.99 \\
    & 200\% & 1.00 & .990 &  99.9\% &  99.8\% &    317.14 &     18.79 &      3.85 \\
\bottomrule
\end{tabular}
\\
\small
\setlength{\tabcolsep}{2.0pt}
\begin{tabular}{@{\extracolsep{3pt}}llllllrrr}
\multicolumn{9}{c}{QM9 8K Subset Polarizability} \\
\toprule
\multicolumn{2}{c}{\thead[tc]{Label Noise}} & \multicolumn{4}{c}{\thead[tc]{Noisy Label Detection\\{\scriptsize\mdseries Thresholding on $\sigma_i$}}} & \multicolumn{3}{c}{\thead[tc]{GPR Accuracy\\{\scriptsize\mdseries 5-fold CV MAE (\AA\textsuperscript{3})}}}\\ \cline{1-2} \cline{3-6} \cline{7-9}
\multirow{2}{*}{\thead{Rate}} & \multirow{2}{*}{\thead{Level}} & \multirow{2}{*}{\thead{R\textsuperscript{2}}} & \multirow{2}{*}{\thead{AUC}} & \multicolumn{2}{c}{\thead[tc]{Precision\\{\mdseries\scriptsize at Recall Level}}} & \multicolumn{3}{c}{\thead[tc]{\mdseries \scriptsize (Pristine Data: 0.53)}}\\ \cline{5-6}
& & & & \thead[tl]{70\%} & \thead[tl]{95\%} & \thead[tl]{Plain} & \thead[tl]{Basic} & \thead[tl]{Full} \\
\midrule
\multirow{4}{*}{10\%} & 10\% & -19.69 & .660 &  16.0\% &  11.1\% &      0.54 &      0.67 &      0.62 \\
    & 50\% &   0.26 & .858 &  44.6\% &  16.8\% &      0.71 &      0.80 &      0.63 \\
    & 100\% &   0.83 & .918 &  73.2\% &  22.5\% &      1.08 &      0.94 &      0.62 \\
    & 200\% &   0.95 & .946 &  78.1\% &  33.4\% &      1.84 &      1.09 &      0.62 \\
\cline{1-9}
\multirow{4}{*}{30\%} & 10\% &  -5.94 & .649 &  41.3\% &  31.9\% &      0.56 &      0.69 &      0.63 \\
    & 50\% &   0.70 & .840 &  69.0\% &  38.9\% &      0.96 &      0.89 &      0.65 \\
    & 100\% &   0.94 & .912 &  90.9\% &  49.9\% &      1.68 &      1.05 &      0.65 \\
    & 200\% &   0.98 & .949 &  92.9\% &  62.4\% &      3.28 &      1.39 &      0.66 \\
\cline{1-9}
\multirow{4}{*}{50\%} & 10\% &  -3.89 & .644 &  61.8\% &  51.9\% &      0.57 &      0.70 &      0.65 \\
    & 50\% &   0.83 & .845 &  83.7\% &  59.0\% &      1.16 &      0.97 &      0.66 \\
    & 100\% &   0.95 & .909 &  95.6\% &  69.5\% &      2.13 &      1.18 &      0.70 \\
    & 200\% &   0.99 & .945 &  96.7\% &  76.9\% &      3.96 &      1.42 &      0.70 \\
\cline{1-9}
\multirow{4}{*}{70\%} & 10\% &  -2.83 & .608 &  76.8\% &  71.0\% &      0.59 &      0.72 &      0.67 \\
    & 50\% &   0.83 & .823 &  90.2\% &  74.9\% &      1.30 &      0.99 &      0.76 \\
    & 100\% &   0.95 & .904 &  98.0\% &  83.5\% &      2.46 &      1.24 &      0.78 \\
    & 200\% &   0.99 & .941 &  98.2\% &  87.3\% &      4.67 &      1.51 &      0.77 \\
\cline{1-9}
\multirow{4}{*}{90\%} & 10\% &  -2.42 & .593 &  92.3\% &  90.1\% &      0.61 &      0.73 &      0.70 \\
    & 50\% &   0.76 & .802 &  96.7\% &  91.5\% &      1.46 &      1.04 &      0.99 \\
    & 100\% &   0.91 & .867 &  99.0\% &  93.1\% &      2.79 &      1.31 &      1.14 \\
    & 200\% &   0.97 & .925 &  99.4\% &  95.3\% &      5.56 &      1.66 &      1.16 \\
\bottomrule
\end{tabular}
\hfill
\small
\setlength{\tabcolsep}{2.0pt}
\begin{tabular}{@{\extracolsep{3pt}}llllllrrr}
\multicolumn{9}{c}{QM9 8K Subset Band Gap} \\
\toprule
\multicolumn{2}{c}{\thead[tc]{Label Noise}} & \multicolumn{4}{c}{\thead[tc]{Noisy Label Detection\\{\scriptsize\mdseries Thresholding on $\sigma_i$}}} & \multicolumn{3}{c}{\thead[tc]{GPR Accuracy\\{\scriptsize\mdseries 5-fold CV MAE (eV)}}}\\ \cline{1-2} \cline{3-6} \cline{7-9}
\multirow{2}{*}{\thead{Rate}} & \multirow{2}{*}{\thead{Level}} & \multirow{2}{*}{\thead{R\textsuperscript{2}}} & \multirow{2}{*}{\thead{AUC}} & \multicolumn{2}{c}{\thead[tc]{Precision\\{\mdseries\scriptsize at Recall Level}}} & \multicolumn{3}{c}{\thead[tc]{\mdseries \scriptsize (Pristine Data: 0.30)}}\\ \cline{5-6}
& & & & \thead[tl]{70\%} & \thead[tl]{95\%} & \thead[tl]{Plain} & \thead[tl]{Basic} & \thead[tl]{Full} \\
\midrule
\multirow{4}{*}{10\%} & 10\% & -61.36 & .532 &  10.6\% &  10.1\% &      0.30 &      0.30 &      0.28 \\
    & 50\% &  -1.85 & .728 &  18.5\% &  10.7\% &      0.32 &      0.32 &      0.28 \\
    & 100\% &   0.30 & .822 &  35.2\% &  11.9\% &      0.36 &      0.34 &      0.29 \\
    & 200\% &   0.82 & .883 &  58.3\% &  13.2\% &      0.48 &      0.40 &      0.29 \\
\cline{1-9}
\multirow{4}{*}{30\%} & 10\% & -24.46 & .531 &  31.1\% &  30.3\% &      0.30 &      0.30 &      0.28 \\
    & 50\% &  -0.28 & .706 &  42.7\% &  31.7\% &      0.34 &      0.34 &      0.30 \\
    & 100\% &   0.65 & .809 &  65.3\% &  33.2\% &      0.44 &      0.38 &      0.30 \\
    & 200\% &   0.92 & .877 &  83.0\% &  36.9\% &      0.74 &      0.46 &      0.31 \\
\cline{1-9}
\multirow{4}{*}{50\%} & 10\% & -17.56 & .530 &  51.0\% &  50.2\% &      0.30 &      0.30 &      0.28 \\
    & 50\% &  -0.04 & .692 &  61.4\% &  51.6\% &      0.37 &      0.35 &      0.31 \\
    & 100\% &   0.68 & .790 &  77.4\% &  53.4\% &      0.51 &      0.41 &      0.33 \\
    & 200\% &   0.92 & .865 &  89.9\% &  57.4\% &      0.89 &      0.49 &      0.34 \\
\cline{1-9}
\multirow{4}{*}{70\%} & 10\% & -15.13 & .528 &  70.9\% &  69.9\% &      0.30 &      0.30 &      0.28 \\
    & 50\% &  -0.10 & .675 &  78.2\% &  71.0\% &      0.39 &      0.36 &      0.34 \\
    & 100\% &   0.60 & .753 &  85.6\% &  72.3\% &      0.59 &      0.43 &      0.39 \\
    & 200\% &   0.88 & .823 &  93.3\% &  74.5\% &      1.06 &      0.53 &      0.42 \\
\cline{1-9}
\multirow{4}{*}{90\%} & 10\% & -15.16 & .517 &  90.4\% &  90.0\% &      0.30 &      0.30 &      0.29 \\
    & 50\% &  -0.35 & .648 &  92.5\% &  90.3\% &      0.41 &      0.37 &      0.37 \\
    & 100\% &   0.41 & .721 &  94.9\% &  90.8\% &      0.63 &      0.44 &      0.47 \\
    & 200\% &   0.78 & .773 &  97.5\% &  91.3\% &      1.16 &      0.53 &      0.59 \\
\bottomrule
\end{tabular}
\end{table*}

We demonstrate the capability of the proposed method using the QM7~\cite{blum970MillionDruglike2009,Rupp2012} and QM9~\cite{ramakrishnanQuantumChemistryStructures2014} data sets. The data sets consists of minimum-energy 3D geometry of small organic molecules and their associated properties such as band gap, polarizability, and heat capacity, \textit{etc}. A marginalized graph kernel is used as the covariance function between graph representations that encode the spatial arrangement and topology of the molecules~\cite{tangPredictionAtomizationEnergy2019}. To carry out the experiments, we add normally distributed noise to the labels in the data sets with a series of combinations of noise rates and noise levels. Here, noise rate is defined as the percentage of labels that we will artificially corrupt, while noise level is the ratio between the noise and the standard deviation of the pristine labels.

From \Cref{table:qm7}, we can see that our proposed method can consistently improve the accuracy of the trained GPR model even in the presence of very high noise rate. Moreover, our method can also capture a high fraction of the noisy labels in most scenarios. Generally speaking, the ability of the method to distinguish noisy and clean labels generally increases with noise level but decreases with noise rate. From an adversarial point of view, this indicates that large numbers of small perturbations within the model's confidence interval is likely to cause the most degradation to the performance of the trained model.

\section{Conclusion}

A method that uses Gaussian process regression to identify noisy real-valued labels is introduced to improve models trained on data sets of high output noise rate. To infer the magnitude of noise on a per-label basis, we use maximum likelihood estimation to optimize a heteroscedastic noise model and learn a element-wise Tikhonov regularization term. We show that this is closely related to denoising using leave-one-out cross-validation. A simple multiplicative updating scheme is designed to solve the numerical optimization problem. The scheme monotonically converges in a broad range of test cases and has defeated our extensive efforts in seeking a counterexample. The capability of the noise detection method is demonstrated on both synthetic examples and real-world scientific data sets.

While we presented a preliminary analysis of the multiplicative update scheme's convergence behavior for a special case, future work is necessary for a thorough determination of the algorithm's region of convergence. There is also strong practical interest in the adaptation of the method to GPR extensions such as those based on non-Gaussian likelihoods and Nystr\"om~\cite{NIPS2000_19de10ad,liMakingLargescaleNystrom2010} or hierarchical low-rank approximations~\cite{liuParallelHierarchicalBlocked2020,rebrovaStudyClusteringTechniques2018a,chavezScalableMemoryEfficientKernel2020}, as well as in multi-level optimizers that can take advantage of the fast convergence of the multiplicative algorithm for the joint optimization of both the kernel hyperparameters and the noise model.

\section{Acknowledgment}
This work was supported by the Luis W. Alvarez Postdoctoral Fellowship at Lawrence Berkeley National Laboratory.

\bibliography{bibliography}
\bibliographystyle{bibstyle}

\clearpage

\appendix
\numberwithin{equation}{section}
\renewcommand{\theequation}{\thesection.\arabic{equation}}

\paragraph{Notations} The Einstein summation convention for repeating indices is implied for all variables except for $i$ and $j$.

\section{Gradient of the negative log-likelihood function}
\label{section:gradient-of-negative-log-likelihood-function}

Starting with the negative log-likelihood of the Gaussian process in \Cref{eq:negative-log-likelihood}, we can derive its partial derivative with respect to $\BB{\Sigma}$:
\begin{align}
\frac{
    \D \LL
}{
    \D \BB{\Sigma}
}
&= \frac{
    \D \log \Det{\td{\B{K}}}
}{
    \D \BB{\Sigma}
} + \frac{
    \D \left( \B{y}^\T \td{\B{K}}^{-1} \B{y} \right)
}{
    \D \BB{\Sigma}
} \label{eq:dL-dSigma-initial-appendix} \\
&= \mathrm{tr}\left( \td{\B{K}}^{-1} \frac{\D \td{\B{K}}}{\D \BB{\Sigma}} \right)
- \left( \td{\B{K}}^{-1} \B{y} \right)^\T \frac{\D \td{\B{K}}}{\D \BB{\Sigma}} \left( \td{\B{K}}^{-1} \B{y} \right) \\
&= \td{\B{K}}^{-1}{}_{kl} \frac{\D \td{\B{K}}_{kl}}{\D \BB{\Sigma}_{ij}}
- \left( \td{\B{K}}^{-1} \B{y} \right)_k \frac{\D \td{\B{K}}_{kl}}{\D \BB{\Sigma}_{ij}} \left( \td{\B{K}}^{-1} \B{y} \right)_l \\
&= \td{\B{K}}^{-1}{}_{kl} \frac{\D \BB{\Sigma}_{kl}}{\D \BB{\Sigma}_{ij}}
- \left( \td{\B{K}}^{-1} \B{y} \right)_k \frac{\D \BB{\Sigma}}{\D \BB{\Sigma}_{ij}} \left( \td{\B{K}}^{-1} \B{y} \right)_l \\
&= \td{\B{K}}^{-1}{}_{kl} \delta_{ki} \delta_{lj}
- \left( \td{\B{K}}^{-1} \B{y} \right)_k \delta_{ki} \delta_{lj} \left( \td{\B{K}}^{-1} \B{y} \right)_l \\
&= \td{\B{K}}^{-1}{}_{ij}
- \left( \td{\B{K}}^{-1} \B{y} \right)_i \left( \td{\B{K}}^{-1} \B{y} \right)_j \\
&= \td{\B{K}}^{-1} - \left( \td{\B{K}}^{-1} \B{y} \right) \left( \td{\B{K}}^{-1} \B{y} \right)^\T \label{eq:dL-dSigma-full-appendix}
\end{align}

\end{document}